\documentclass[conference]{IEEEtran}

\usepackage{cite}
\usepackage{amsmath,amssymb,amsfonts}
\usepackage{amsthm}
\usepackage{algorithmic}
\usepackage{graphicx}
\usepackage{textcomp}
\usepackage{cleveref}
\usepackage{cmll}
\usepackage{balance}

\usepackage{booktabs}
\newtheorem{definition}{Definition}
\newtheorem{theorem}{Theorem}
\newtheorem{remark}{Remark}

\newcommand{\WBF}{\widehat{\diamond}}
\newcommand{\CC}{\Subset}
\newcommand{\bel}{\mathbf{b}}       
\newcommand{\base}{\mathbf{a}}        
\newcommand{\evidence}{\mathbf{r}}
\newcommand{\prob}{\mathbf{p}}
\newcommand{\Dir}{\text{Dir}}

\begin{document}

\title{Multi-Source Fusion Operations in Subjective Logic}

\author{\IEEEauthorblockN{Rens W. van der Heijden}
\IEEEauthorblockA{\textit{Institute of Distributed Systems} \\
\textit{Ulm University}\\
Ulm, Germany\\
\texttt{rens.vanderheijden@uni-ulm.de}}
\and
\IEEEauthorblockN{Henning Kopp}
\IEEEauthorblockA{\textit{Institute of Distributed Systems} \\
\textit{Ulm University}\\
Ulm, Germany\\
\texttt{henning.kopp@uni-ulm.de}}
\and
\IEEEauthorblockN{Frank Kargl}
\IEEEauthorblockA{\textit{Institute of Distributed Systems} \\
\textit{Ulm University}\\
Ulm, Germany\\
\texttt{frank.kargl@uni-ulm.de}}
}

\maketitle

\begin{abstract}
The purpose of multi-source fusion is to combine information from more than two evidence sources, or subjective opinions from multiple actors.
For subjective logic, a number of different fusion operators have been proposed, each matching a fusion scenario with different assumptions.
However, not all of these operators are associative, and therefore multi-source fusion is not well-defined for these settings.
In this paper, we address this challenge, and define multi-source fusion for weighted belief fusion (WBF) and consensus \& compromise fusion (CCF).
For WBF, we show the definition to be equivalent to the intuitive formulation under the bijective mapping between subjective logic and Dirichlet evidence PDFs.
For CCF, since there is no independent generalization, we show that the resulting multi-source fusion produces valid opinions, and explain why our generalization is sound.
For completeness, we also provide corrections to previous results for averaging and cumulative belief fusion (ABF and CBF), as well as belief constraint fusion (BCF), which is an extension of Dempster's rule.
With our generalizations of fusion operators, fusing information from multiple sources is now well-defined for all different fusion types defined in subjective logic.
This enables wider applicability of subjective logic in applications where multiple actors interact.
\end{abstract}


\section{Introduction}

Subjective logic is an extension of probabilistic logic that enables the separate representation of \emph{belief} and \emph{uncertainty}.
The belief mass assignment and uncertainty are represented in a \emph{subjective opinion} $\omega^A_X$ held by an evidence source or \emph{actor} $A$ on a random variable $X$.
Similar to other extensions introduced throughout the literature, such as Dempster-Shafer Theory (DST)~\cite{Shafer-DST,Dempster-DST}, such extensions are widely used for the purpose of data fusion.
In data fusion, observations of the same variable from multiple \emph{evidence sources} are fused by applying a fusion operator, such as Dempster's rule of combination in Dempster-Shafer Theory.
Depending on the specific application, an evidence source could be a sensor in an autonomous vehicle, information in a court case, or an expression of trust from two actors in a peer-to-peer system.
These different applications require fusion operators with different properties, such as dependence between the evidence seen by evidence sources. 
Subjective logic provides a wide variety of these operators~\cite[Ch. 12]{Josang-SL-book}, including \emph{averaging belief fusion}, \emph{cumulative belief fusion}, \emph{weighted belief fusion}, \emph{consensus \& compromise fusion} and \emph{belief constraint fusion}.
Each of these fusion operations is designed to determine the shared belief and uncertainty of a group of evidence sources, with different applications depending on how evidence should be combined.
For example, cumulative belief fusion can be thought of as adding together the evidence for each potential value of the variable $X$ from all sources and representing this as an opinion.
In contrast, averaging belief fusion computes the average of the evidence for each potential value.
These two fusion operations, along with weighted belief fusion, have direct analogs in the Dirichlet model that considers evidence.

Not all of the operators defined for subjective logic are associative, meaning that the fusion between more than two evidence sources is in general not well-defined.
However, as pointed out by J{\o}sang~\cite[Ch. 12]{Josang-SL-book}, the fusion operators have semi-associative properties, which roughly means that the individual steps of the fusion process are associative, but at the end of this process, the result is normalized.
J{\o}sang and co-authors have since proposed multi-source formulations of averaging and cumulative belief fusion in a previous work~\cite{Josang-Multi-Source}.
In this paper, we contribute further multi-source fusion formulations for weighted belief fusion (WBF) and consensus \& compromise fusion (CCF).
We also discuss a minor correction to the corner cases of ABF and CBF by J{\o}sang~et~al.~\cite{Josang-Multi-Source}.

The remainder of this work is organized as follows.
Section \ref{sec:opinions} introduces the theoretical background of subjective opinions and their relationship with Dirichlet PDFs.
Section \ref{sec:multi-source-fusion} formulates multi-source variants of WBF and CCF.
For WBF, we prove that our formulation is equivalent to the confidence-weighted averaging of the Dirichlet evidence PDFs.
For CCF, there is no direct equivalence: we only show that the multi-source formulation is consistent, generates valid opinions, and is conceptually correct.
Section \ref{sec:example} extends the example from~\cite{Josang-Multi-Source} with numerical results for the new operations and potential applications.
Section \ref{sec:conclusion} concludes our work.


\section{Subjective Opinions}
\label{sec:opinions}

This section introduces the mathematical foundation of subjective logic and the hyper-Dirichlet model, which is an extension of the Dirichlet multinomial model.
In addition, we introduce the respective notation in both settings, provide some intuition with regards to the interpretation, and describe a bijective map between subjective logic and the hyper-Dirichlet model.

\subsection{Multinomial Opinions}

Multinomial opinions~\cite[Ch 3.5]{Josang-SL-book} express belief
$\bel_X(x)\in [0,1]$ over the possible values $x$ of a
random variable $X$, as well as an uncertainty $u_X$ and a base rate
$\base_X(x)\in[0,1]$ for every possible value.
The conceptual purpose of multinomial opinions is to describe the available evidence or belief in each possible value $x\in\mathbb{X}$ in the domain of $\mathbb{X}$.
We first provide a formal definition of multinomial opinions, before describing their application and interpretation with an example.

\begin{definition}[Multinomial Opinion]\label{d_mopinion}
  Let $X\in\mathbb{X}$ be a random variable over the finite domain $\mathbb{X}$. A multinomial opinion $\omega^{A}_{X}$ held by $A$ over $X$ describes the subjective assignment of belief by $A$ to the outcome of $X$, consisting of an ordered triplet $\omega^{A}_{X} = (\bel^{A}_{X}, u^{A}_{X}, \base^{A}_{X})$. Here, $\bel^{A}_{X}: \mathbb{X}\rightarrow[0,1]$ is the belief mass distribution over $\mathbb{X}$, $u^{A}_X\in[0,1]$ represents the lack of evidence, and $\base^{A}_{X}:\mathbb{X}\rightarrow[0,1]$ is the base rate distribution, with the additivity requirements that $1 = u^{A}_{X} + \sum_{x\in\mathbb{X}}\bel^{A}_{X}(x)$ and $\sum_{x\in\mathbb{X}} \base^{A}_{X}(x) = 1$.
\end{definition}

If the random variable $X$ is the outcome of tossing a potentially weighted die that uses colors instead of numbers, then its domain $\mathbb{X}$ contains the possible color values (e.g., red, green, blue, and so on).
A multinomial opinion then assigns a \emph{belief mass} to each color.
Initially, the uncertainty $u_X$ of this distribution can be set to $1$, which is referred to as a \emph{vacuous opinion}.
The base rate of this opinion is the assumption that the die is not weighted, i.e., that all outcomes are equally likely: $a(x)=1/|\mathbb{X}| = 1/6$ for all $x\in\mathbb{X}$.
It is also possible to have opinions that assign all the available mass to belief, i.e., $u_X=0$, which is referred to as a \emph{dogmatic opinion}.
This represents an edge case, since by definition zero uncertainty means that there is absolute certainty in the belief distribution, meaning the base rate has no bearing on the outcome at all.

To see how multinomial opinions can be applied, consider the die example above.
By using opinions as evidence accumulators, each sampling of the random variable (i.e., throwing the die and adjusting the belief) increases the available evidence and thus the corresponding belief in the given outcome.
The expected outcome of our opinion projects the belief and the uncertainty to a probability $P_X(x)$ by distributing the uncertainty over the potential values in $\mathbb{X}$ according to the base rate $a(x)$.
More formally~\cite[Ch. 3]{Josang-SL-book}:
\[
  P_X(x) = \bel_X(x) + \base_X(x) \cdot u_X
\]

Multinomial opinions are often useful to represent evidence from a set of different sources $\mathbb{A}$, as common in trust management and data fusion.
Referring again to the die example above: if the observers have some kind of color blindness, or the lighting conditions are sub-optimal, one can imagine that different observers have different opinions about the same observation.
This is where multinomial opinions can be used by these observers to agree on the outcome of specific observations, as well as estimation of the real properties of the die (by observing repeated experiments).
These operations can be achieved with fusion operations, as described in Section \ref{sec:multi-source-fusion}.
One challenge for multinomial opinions is that it is not possible to represent an observer that cannot distinguish two colors in the above example (e.g., red-green color blindness), other than assigning the same amount of evidence to both values.
However, this assignment implies that the relative frequency of red and green is the same, while in fact what we want to represent is the inability to distinguish between these variables.
This is where hyper-opinions are useful.

\subsection{Hyper Opinions}

Hyper opinions~\cite[Ch. 3.6]{Josang-SL-book} are the natural extension of multinomial opinions, which allow belief assignment to composite values.
For example, taking $x_1,x_2,x_3\in\mathbb{X}$, a multinomial opinion can assign a belief $b(x_1)=0.1$, $b(x_2)=0.3$, $b(x_3)=0.4$, indicating the relative evidence between these three possible values.
Some situations require the assignment of belief to composite values, e.g., $b(x_1 \cup x_2)=0.4$, stating that there is evidence for either $x_1$ or $x_2$, without specifying the relative evidence between them.
This is exactly the example we gave above with red-green color blindness, with $x_1=\text{red}$ and $x_2=\text{green}$.
Another such example was given by Hankin~\cite{Hankin2010} in the context of Dirichlet hyper-PDFs, who uses the example of relative strengths of tennis players in a tournament.
In their example, the outcome $X=x_i$ means $i$ wins a tournament, and the belief assignment distributes the available evidence over the players.
Assigning belief to $x_1\cup x_2$ represents evidence that player 1 or 2 will win, but does not make statements about the relative probabilities between player 1 and 2.

The formal description of a hyper opinion is then:

\begin{definition}[Hyper Opinion]\label{d_hopinion}
  Let $X\in\mathbb{X}$ be a random variable over the finite domain $\mathbb{X}$, with the power set $\mathcal{P}(\mathbb{X})$ and the reduced power set $\mathcal{R}(\mathbb{X})=\mathcal{P}(\mathbb{X}) \setminus \{\mathbb{X}, \emptyset\}$. A hyper opinion $\omega^{A}_{X}$ held by $A$ over $X$ describes the subjective assignment of belief by $A$ to the outcome of $X$, consisting of an ordered triplet $\omega^{A}_{X} = (\bel^{A}_{X}, u^{A}_{X}, \base^{A}_{X})$. Here, $\bel^{A}_{X}:\mathcal{R}(\mathbb{X})\rightarrow[0,1]$ is the belief mass distribution over $\mathcal{R}(\mathbb{X})$, $u^{A}_X\in[0,1]$ represents the lack of evidence, and $\base^{A}_{X}:\mathbb{X}\rightarrow[0,1]$ is the base rate distribution over $\mathbb{X}$, with the additivity requirements that $1 = u^{A}_{X} + \sum_{x\in\mathcal{R}(\mathbb{X})}\bel^{A}_{X}(x)$ and $\sum_{x\in\mathcal{R}(\mathbb{X})} \base^{A}_{X}(x) = 1$.
\end{definition}

\begin{remark}[Base rate over singleton values]
Note that the base rate distribution is still over $X$ here, meaning that $X$ takes values from $\mathbb{X}$, not from $\mathcal{P}(\mathbb{X})$, and thus a hyper opinion has $(2^k+k-3)$ degrees of freedom for a domain $\mathbb{X}$ of cardinality $k$, as discussed in~\cite[Ch. 3, p. 40]{Josang-SL-book}.
\end{remark}

\begin{figure*}[ht!]
\fbox{\parbox{\textwidth}{
    \[
        \left(\begin{array}{lcl}
          \bel^A_X(x) &=& \frac{\evidence^A_X(x)}{W + \sum\limits_{x\in\mathcal{R}(\mathbb{X})}\evidence^{A}_X(x)} \\
            u^A_X&=&\frac{W}{W + \sum\limits_{x\in\mathcal{R}(\mathbb{X})}\evidence^{A}_X(x)}
        
        \end{array}
        \right)
        \stackrel{\text{1-to-1}}{\longleftrightarrow}
        \left(\begin{array}{lrcl}
            \underline{\text{For $u^{A}_X \neq 0$:}} & \evidence^A_X(x) &=& \frac{W\bel^A_X(x)}{u^A_X} \\
                                             & 1 &=& u^A_X + \sum\limits_{x\in\mathcal{R}(\mathbb{X})}\bel^A_X(x) \\
            \underline{\text{For $u^{A}_X = 0$:}}    & \evidence^A_X(x) &=& \bel^A_X(x)\cdot \infty \\
                                             & 1 &=& \sum\limits_{x\in\mathcal{R}(\mathbb{X})}\bel^A_X(x)
        \end{array}
        \right)
    \]
}}
  \caption{Bijection between hyper opinions (left) and Dirichlet HPDFs (right), as defined in~\cite{Josang-SL-book}.}\label{f_bijectiondirichletopinion}
\end{figure*}
\subsection{Dirichlet Distributions}

Both multinomial and hyper opinions are directly related to Dirichlet distributions.
The traditional Dirichlet distribution is defined by a vector of $k$ parameters, $\mathbf{\alpha}_X$, resulting in the notation $\Dir(\prob_X, \alpha_X)$.
This model has wide statistical applications; for our purposes, it will ground interpretations of opinions in a statistical sense.
The probability density function (PDF) of a Dirichlet distribution can be adapted to an evidence-based formulation by setting the vector of parameters to represent the evidence for the possible outcomes of $X$.
This is done by selecting a non-informative prior weight $W$ and configuring the strength parameters $\alpha_X$ to be $\evidence_X+\base_X W$, where $\evidence_X(x)\ge 0$ for all $x\in\mathbb{X}$.
The PDF is then denoted $\Dir^e_{X}(\prob_X, \evidence_X, \base_X)$, where $\evidence_X$ is termed the evidence vector and $\base_X$ is the base rate distribution.
The non-informative prior weight is set to $W=2$, resulting in a uniform base rate distribution, as discussed in previous work~\cite{Josang-Multi-Source,Josang-SL-book}.
The PDF is then written as~(compare~\cite[Ch. 3, p. 32-33]{Josang-SL-book}):

\begin{multline}
\Dir^e_{X}(\prob_X, \evidence_X, \base_X) = \frac{
  \Gamma\left(\sum_{x\in\mathbb{X}}(\evidence_X(x)+\base_X(x)W) \right)
}{
  \prod_{x\in\mathbb{X}} \Gamma(\evidence_X(x)+\base_X(x)W)
}\cdot \\
  \prod_{x\in\mathbb{X}} \prob_X(x)^{\evidence_X(x)+\base_X(x)W-1}
\end{multline}

where $\evidence_X(x) + \base_X(x)W \ge 0$ and $\prob_X(x) \neq 0$ if $\evidence_X(x)+\base_X(x)W < 1$.

This is the multinomial Dirichlet model, as also discussed in~\cite{Josang-Multi-Source}.
Hankin~\cite{Hankin2010} formulated an extension of this model that allows evidence to be assigned to composite values, similar to hyper opinions, referred to as Dirichlet hyper-PDFs.
J{\o}sang~\cite[Sec. 3.6.3]{Josang-SL-book} defines the evidence-based formulation of those hyper-PDFs, summarized as follows (where the superscript $H$ indicates the hyper probability distribution is concerned):

\begin{multline}
  \Dir^{eH}_{X}(\prob^H_X, \evidence_X, \base_X) = \frac{
    \Gamma\left(\sum_{x\in\mathcal{R}(\mathbb{X})}(\evidence_X(x)+\base_X(x)W) \right)
}{
  \prod_{x\in\mathcal{R}(\mathbb{X})} \Gamma(\evidence_X(x)+\base_X(x)W)
}\cdot \\
  \prod_{x\in\mathcal{R}(\mathbb{X})} \prob^H_X(x)^{\evidence_X(x)+\base_X(x)W-1}
\end{multline}

where $\evidence_X(x) + \base_X(x)W \ge 0$ and $\prob^H_X(x) \neq 0$ if $\evidence_X(x)+\base_X(x)W < 1$.

J{\o}sang~\cite[Sec. 3.6.5]{Josang-SL-book} also describes the relationship between the PDF of the hyper probability distribution $p^H_X$ and the probability distribution $p_X$: we refer interested readers to this work for a detailed discussion of the interpretative differences between these cases.
For this work, it suffices that there is a well-defined grounding of both multinomial and hyper opinions in the Dirichlet multinomial model and the Dirichlet hyper-PDF model, respectively.

\subsection{Mapping Opinions to Dirichlet (H)PDFs}
\label{s_bijection}

We now introduce the relation between opinions and Dirichlet (H)PDFs.
There is a bijective map between Dirichlet HPDFs of the form $\Dir^{eH,A}_X(\prob^{H}_X,\evidence^{A}_X,\base^{A}_X)$ and hyper opinions $\omega^{A}_X=(\bel^{A}_X,u^{A}_X,\base^{A}_X)$, as shown in Figure~\ref{f_bijectiondirichletopinion}~(compare~\cite[Def. 3.9]{Josang-SL-book}).
A similar map can be defined between multinomial opinions and Dirichlet PDFs, where the range of the mapping is constrained to $x\in\mathbb{X}$ instead of $x\in\mathcal{R}(\mathbb{X})$.
Equivalently, for multinomial opinions and Dirichlet PDFs, the bijective map in Figure~\ref{f_bijectiondirichletopinion} applies where $\evidence$ and $\bel$ for all composite values is defined as $0$.

For both the multinomial and the hyper opinion setting, there is an important edge case in this bijective map when the uncertainty is $0$.
As previously discussed, this semantically means that the belief distribution is ``known'', which implies there is infinite evidence.
We briefly discuss the intuition behind this concept, as it is essential for the corresponding edge cases in the fusion operations, both in previous work~\cite{Josang-Multi-Source} and in our formulations below.
This is equivalent to taking the limit $\lim_{u^A_X\rightarrow0}$ of the case for $u^A_X\neq0$, giving the expression in Figure \ref{f_bijectiondirichletopinion}.
To distinguish between different evidence variables' frequencies, the notion of relative degrees of infinity is used~\cite{Josang-belief-fusion}.
This intuitively corresponds to the relative likelihood of the different values for the variable $X$.
For a detailed discussion of these conceptual notions, we refer interested readers to earlier work by J{\o}sang~et~al.~\cite[Sec. 4]{Josang-belief-fusion}.

\section{Multi-source Fusion}
\label{sec:multi-source-fusion}

Multi-source fusion refers to the combination of opinions from multiple sources.
Intuitively, this can be understood as a set of actors $\mathbb{A}$ coming together to agree on a common conclusion ($\omega^{\circ\mathbb{A}}_X$), using some fusion operator $\circ$.
This operator specifies precisely how to combine the information represented by each actor's subjective opinion.
For operators that are both commutative and associative, $\omega^{A_1}_X\circ\omega^{A_2}_X\circ\omega^{A_3}_X\dots$ is well-defined, and therefore fusion is straight-forward.
However, the fusion operators defined as binary operators in~\cite[Ch. 12]{Josang-SL-book} are not all associative, i.e., in general, $\omega^{A_1}_X\circ(\omega^{A_2}_X\circ\omega^{A_3}_X)\neq(\omega^{A_1}_X\circ\omega^{A_2}_X)\circ\omega^{A_3}_X$, motivating the need to generalize the operators to accept arrays of opinions that are merged through fusion.
This section defines multi-source fusion operations to generalize the operators not discussed in~\cite{Josang-Multi-Source}.

\subsection{Cumulative and Averaging Belief Fusion}
\label{sec:CBF-ABF}

In previous work, the authors~\cite{Josang-Multi-Source} propose multi-source fusion for the CBF and ABF operators.
The definitions provided by the authors for CBF (Equations 16-19) and ABF (Equations 32-35) require some adjustments: the cases defined in the original paper are incorrect and lead to division by zero.
Division by zero occurs when at least one opinion has a non-zero uncertainty, while at least two (other) opinions have zero uncertainty: this leads to Case I in the original formulation, but in the corresponding equations, all the products in the divisions for $\bel$ and $u$ are zero.
Therefore, a straight-forward correction is to require that \emph{all} uncertainties to be zero whenever applying Case I.
One should then also modify the condition for Case II accordingly, and additionally modify the equations such that any non-dogmatic opinions are discarded.
This makes sense intuitively, since any non-dogmatic opinion represents finite evidence in the Dirichlet evidence PDF representation (through the bijective map of Figure \ref{f_bijectiondirichletopinion}), while a dogmatic opinion represents infinite evidence.
Therefore, Case II discards finite evidence and combines the dogmatic opinions only.
The original authors have since published a revised version of the paper on-line~\footnote{\texttt{https://folk.uio.no/josang/papers/JWZ2017-FUSION.pdf}}.

%
%
%
%
%
%

\subsubsection{Epistemic Cumulative Belief Fusion}

Although it was not explicitly discussed in~\cite{Josang-Multi-Source}, a multi-source variant of epistemic cumulative belief fusion can also be derived from the aleatory form discussed above.
The difference between aleatory and epistemic fusion is the type of knowledge represented by the opinion; the aleatory variant describes the fusion of information with respect to specific observations, while the epistemic variant describes the fusion of knowledge.
For epistemic cumulative belief fusion (e-CBF) of multiple opinions (as opposed to aleatory CBF), the only reasonable approach is then to first apply the multi-source aleatory cumulative fusion described in~\cite{Josang-Multi-Source}, until all relevant epistemic knowledge has been cumulated. 
Only when a decision needs to be made should uncertainty maximization (as discussed in~\cite[Ch. 12]{Josang-SL-book}) be applied to the resulting opinion.
To see why this is reasonable, consider that e-CBF is used in an on-line fashion, fusing two opinions together before fusing them with a third.
Fusing an e-CBF result with epistemic knowledge from a third source skews results in favor of this third source, since the uncertainty controls how much of the belief mass of the new epistemic evidence is considered for the final result.
Instead, one would normally expect that the three opinions are considered equally in the fusion process.
Similarly, in practice, an analyst would first combine all the available knowledge (i.e., retrieve all relevant epistemic evidence from all sources), and only then make a decision -- if more information becomes available later on, one would expect the analyst to re-do the computation, rather than fuse the uncertainty-maximized result with the new information.
Thus, indeed, e-CBF should be performed by applying a-CBF to \emph{all} opinions, and then the result should be uncertainty-maximized.

\subsection{Belief Constraint Fusion}

Belief constraint fusion (BCF) is the extension of Dempster's rule of combination to hyper opinions.
Although Dempster's rule of combination is associative, this rule does not consider the base rate, as Dempster-Shafer Theory does not foresee such base rate distinctions.
If the base rate is not the same for all inputs, the operator defined by J{\o}sang~\cite[Ch. 12.2]{Josang-SL-book} is not associative, as pointed out by the author: we thus require a multi-source formulation of this operation.

\begin{definition}[BCF of opinions for multiple sources]
    Let $\mathbb{A}$ be a finite set of actors and let
    $\omega^{A}_{X}= (\bel^{A}_{X}, u^{A}_{X}, \base^{A}_{X})$ denote
    the multinomial opinion held by $A\in\mathbb{A}$ over $X$.
    We define the belief constraint fusion of these opinions as the
    opinion $\omega^{\with \mathbb{A}}_{X}= (\bel^{\with \mathbb{A}}_{X}, u^{\with \mathbb{A}}_{X}, \base^{\with \mathbb{A}}_{X})$ where

\begin{eqnarray*}
  \bel^{\with\mathbb{A}}_{X}(x) &=& m^{\oplus\mathbb{A}}(x)\\
  u^{\with\mathbb{A}}_{X} &=& m^{\oplus\mathbb{A}}(\mathbb{X})\\
  \base^{\with\mathbb{A}}_X (x) &=&
        \begin{cases}
              \frac{\sum\limits_{A\in\mathbb{A}}\base^{A}_{X}(x)\left(1-u^{A}_{X}\right)}{\sum\limits_{A\in\mathbb{A}}\left(1-u^{A}_{X}\right)} &\text{, for } \sum\limits_{A\in\mathbb{A}} u^{A}_{X} < \lvert\mathbb{A}\rvert  \\
              \frac{\sum\limits_{A\in\mathbb{A}}\base^{A}_{X}(x)}{\lvert\mathbb{A}\rvert}     &\text{, otherwise} 
  \end{cases}
\end{eqnarray*}

    Here, the belief mass is merged through Dempster's rule of combination~\cite[Ch. 3.1]{Shafer-DST}:

\begin{eqnarray}
  m^{\oplus\mathbb{A}}(x) &=& (m_1 \oplus m_2 \oplus \hdots\oplus m_{\lvert\mathbb{A}\rvert})(x)
\end{eqnarray}
where
\begin{eqnarray*}
  (m_i \oplus m_j)(\emptyset) &=& 0\\
  (m_i \oplus m_j)(x) &=& \frac{\sum\limits_{y \cap z = x} m_i(y)m_j(z)}{1-K_{i,j}} \text{, for $x\neq\emptyset$}\\
\end{eqnarray*}
and
\[
  K_{i,j} = \sum\limits_{y \cap z = \emptyset} m_i(y)m_j(z)
\]
\end{definition}

This definition essentially corresponds to the application of a map between the individual opinions to DST, applying Dempster's rule, and mapping back again.
Although it is obvious that this is well-defined for the belief and uncertainty (since the translation is one-to-one~\cite[Ch. 5]{Josang-SL-book}), the base rate needs separate consideration.
The special case where all the base rates are the same, as implicitly assumed by DST (i.e., that $\base_X(x)=1/\lvert\mathbb{X}\rvert$), also holds trivially.
However, subjective logic technically allows for different base rates: these are combined as a confidence-weighted average (first case in the definition above), as defined by J{\o}sang~\cite[Ch. 12]{Josang-SL-book}.
This only works if there is at least one non-vacuous opinion (i.e., $\exists x\in\mathbb{X},A\in\mathbb{A}:b^A_X(x)\neq0$): if all confidences are zero, the confidence is equal for all opinions, and therefore the most meaningful combination of potentially distinct base rates is computing their average.
Note that both cases preserve the standard setting, where base rates are equal across all opinions (including the output).

\subsection{Weighted Belief Fusion}

Audun J{\o}sang defined the weighted belief fusion of two opinions in~\cite[Def. 12.8, p. 232]{Josang-SL-book}, which we extend here to a multi-source variant as follows:

\begin{definition}[WBF of opinions for multiple sources]\label{d_wbf}
    Let $\mathbb{A}$ be a finite set of actors and let
    $\omega^{A}_{X}= (\bel^{A}_{X}, u^{A}_{X}, \base^{A}_{X})$ denote
    the multinomial opinion held by $A\in\mathbb{A}$ over $X$.
    Then we define the weighted belief fusion of these opinions as the
    opinion $\omega^{\WBF \mathbb{A}}_{X}= (\bel^{\WBF \mathbb{A}}_{X}, u^{\WBF \mathbb{A}}_{X}, \base^{\WBF \mathbb{A}}_{X})$
    as follows:

    \underline{Case 1:} $(\forall A\in\mathbb{A}: u^{A}_X \neq 0) \wedge (\exists A\in\mathbb{A}: u^{A}_X \neq 1)$

\begin{eqnarray*}
  \bel^{\WBF \mathbb{A}}_X(x) &=& \frac{\sum\limits_{A\in\mathbb{A}} \bel^{A}_X(x) \left(1 - u^{A}_X\right) \prod\limits_{A' \in \mathbb{A}, A' \neq A} u^{A'}_X}{\left(\sum\limits_{A\in\mathbb{A}}\prod\limits_{A'\neq A} u^{A}_{X}\right) - \lvert\mathbb{A}\rvert\cdot\prod\limits_{A \in \mathbb{A}} u^{A}_X}\\
  u^{\WBF \mathbb{A}}_X &=& \frac{\left(\lvert\mathbb{A}\rvert - \sum\limits_{A \in \mathbb{A}} u^{A}_X\right)\cdot\prod\limits_{A \in \mathbb{A}} u^{A}_X}{\left(\sum\limits_{A\in\mathbb{A}}\prod\limits_{A'\neq A} u^{A}_{X}\right) - \lvert\mathbb{A}\rvert\cdot\prod\limits_{A \in \mathbb{A}} u^{A}_X} \\
  \base^{\WBF \mathbb{A}}_X(x) &=& \frac{\sum\limits_{A\in\mathbb{A}} \base^{A}_X(x) \left(1 - u^{A}_X\right)}{\lvert\mathbb{A}\rvert - \sum\limits_{A \in \mathbb{A}} u^{A}_X}
\end{eqnarray*}

    \underline{Case 2:} $ \exists A\in\mathbb{A}: u^{A}_X = 0$. Let $\mathbb{A}^{\text{dog}} = \{A\in\mathbb{A}: u^A_X = 0 \}$, i.e., $\mathbb{A}^{\text{dog}}$ is the set of all dogmatic opinions in the input.

\begin{eqnarray*}
  \bel^{\WBF \mathbb{A}}_X(x) &=& \sum\limits_{A\in\mathbb{A}^{\text{dog}}} \gamma^A_X\bel^A_X(x)\\
  u^{\WBF \mathbb{A}}_X &=& 0 \\
  \base^{\WBF \mathbb{A}}_X(x) &=& \sum\limits_{A\in\mathbb{A}^{\text{dog}}} \gamma^A_X\base^A_X(x)\\
      \text{where } \gamma^{A}_{X} &=& \lim\limits_{u^{\WBF \mathbb{A}}_X \rightarrow 0} \frac{u^A_X}{\sum\limits_{A'\in\mathbb{A}^{\text{dog}}}u^{A'}_X}
\end{eqnarray*}

  Note that $\gamma^A_X$ is defined by this limit due to the bijective map (analogous to~\cite[Ch. 12]{Josang-SL-book}~and~\cite{Josang-Multi-Source}).

      \underline{Case 3:} $\forall A\in\mathbb{A}: u^{A}_X = 1$

\begin{eqnarray*}
  \bel^{\WBF \mathbb{A}}_X(x) &=& 0\\
  u^{\WBF \mathbb{A}}_X &=& 1 \\
  \base^{\WBF \mathbb{A}}_X(x) &=& \frac{\sum\limits_{A\in\mathbb{A}} \base^A_X(x)}{|\mathbb{A}|}\\
\end{eqnarray*}
\end{definition}

\begin{remark}[Infinite evidence]
  In Definition \ref{d_wbf}, case 2 describes the resulting combined belief using the relative weight $\gamma^A_X$ per actor with infinite evidence.
  This is analogous to the notion of relative infinities in the bijective map discussed in Section \ref{s_bijection}.
  We exclude all finite evidence parameters here, by only considering the actors with dogmatic opinions: all non-dogmatic actors will have finite (and therefore negligible) evidence.
\end{remark}

\begin{remark}[No evidence]
  If no evidence is available (Definition \ref{d_wbf}, case 3) in any of the inputs, this operation would divide by zero (since $u^A_X=1$ for all actors, making the nominator 0 in case 1).
  However, if no evidence is available, the fused opinion should obviously also have no evidence: this justifies the definition in case 3.
\end{remark}

We remark that confidence-weighted combination of base rates as defined in Definition \ref{d_wbf} maintains the intuitive property that if all base rates of inputs are equal, the output has the same base rate, regardless of the case.
Intuitively, since the base rate represents the belief in absence of information, this should always be the case, but the theory allows actors to have different base rates, in which case they are averaged.

This derivation should intuitively correspond to that derived from the Dirichlet HPDF, as shown in~\cite[Thm. 12.4]{Josang-SL-book} for two opinions.
In short, this theorem derives that the weighted belief fusion operator for two opinions is equivalent to confidence-weighted averaging of the evidence parameters of the two corresponding Dirichlet HPDFs.
We should thus show that our definition corresponds to the confidence-weighted averaging of the evidence parameters of all opinions.
This is defined as:
\begin{eqnarray}
  Dir^{eH}_{X}(\prob^{H}_{X}, \evidence^{\WBF \mathbb{A}}_{X}, \base^{\WBF\mathbb{A}}_{X}), \text{where}\\
  \evidence^{\WBF \mathbb{A}}_{X}(x)=\frac{\sum_{A\in\mathbb{A}} \evidence^{A}_{X}(x)\cdot (1-u^{A}_{X})}{\sum_{A\in\mathbb{A}} (1-u^{A}_X)}\nonumber
\end{eqnarray}

\begin{theorem}
    Our definition of weighted belief fusion of opinions for multiple sources is compatible with weighted belief fusion of Dirichlet HPDFs.
\end{theorem}
\begin{proof}
    We want to show that the WBF of multiple Dirichlet HPDFs
    corresponding to subjective opinions, when mapped back to a
    subjective opinion yield the formulas in~\Cref{d_wbf}.
    We show only the derivation of $u^{\WBF \mathbb{A}}_X$, where
    $u^A_X\neq 0$ for all $A\in\mathbb{A}$.
    The other derivations follow the same basic pattern and do not
    provide any additional insights.

    Using the mapping between Dirichlet HPDFs and hyper opinions
    from~\Cref{f_bijectiondirichletopinion} we compute the evidence
    of multiple fused HPDFs corresponding to hyper opinions as follows.
    \begin{eqnarray}
    \evidence^{\WBF \mathbb{A}}_{X}(x) & = & \frac{\sum\limits_{A\in\mathbb{A}} \evidence^{A}_{X}(x)\cdot (1-u^{A}_{X})}{\sum\limits_{A\in\mathbb{A}} (1-u^{A}_X)} \\
    & = & \frac{\sum\limits_{A\in\mathbb{A}}\left(\bel^{A}(x)(1-u^{A}_X)\cdot\prod\limits_{A'\neq A}u^{A}_X\right)}{\sum\limits_{A\in\mathbb{A}}(1-u^{A}_X) \prod\limits_{A\in\mathbb{A}} u^{A}_X}\nonumber
    \end{eqnarray}

    Again, using the mapping in~\Cref{f_bijectiondirichletopinion} we know that
    the uncertainty of the fused HPDFs is
    \[
        u_X^{\WBF \mathbb{A}}=\frac{W}{W + \sum\limits_{x\in\mathcal{R}(\mathbb{X})}\evidence_X^{\WBF \mathbb{A}}(x)}
    \]
    Plugging in $\evidence^{\WBF \mathbb{A}}_{X}(x)$ we can cancel $W$ and receive
    \begin{multline*}
        u_X^{\WBF \mathbb{A}}=\left(\sum\limits_{A\in\mathbb{A}}(1-u^{A}_X)\prod\limits_{A\in\mathbb{A}} u^{A}_X\right)\cdot \\
        \left(\sum\limits_{A\in\mathbb{A}}(1-u^{A}_X) \prod\limits_{A\in\mathbb{A}} u^{A}_X + \right. \\
        \left.\sum\limits_{x\in\mathcal{R}(\mathbb{X})}\sum\limits_{A\in\mathbb{A}}\bel^{A}(x)(1-u^{A}_X)\prod\limits_{A'\neq A}u^{A'}_X\right)^{-1}
    \end{multline*}

    Next, we can substitute
    $\sum_{x\in\mathcal{R}(\mathbb{X})}\bel^A_X(x)$ by
    $1-u^A_X$.\\
    A tedious but not difficult computation yields the result.
\end{proof}

\subsection{CC Fusion}
%

CC fusion (CCF) is conceptually designed to first achieve consensus, conserving the agreed weight of all inputs, and then compute a weighted compromise for the remaining belief mass based on the relative uncertainty and the corresponding base rates relative to intersecting sets.
Because this second step generally results in a total distribution greater than $1$, a normalization factor $\eta$ is used, which is multiplied with the compromise belief to compute the final fused result.
Let again $\mathbb{A}$ be a finite set of actors and let $\omega^{A}_{X}= (\bel^{A}_{X}, u^{A}_{X}, \base_{X})$ denote the multinomial opinion held by $A\in\mathbb{A}$ over $X$.
Then, CC fusion $\omega_X^{\CC\mathbb{A}}$ is defined as the result of the following three
computation phases: 1) consensus phase, 2) compromise phase, 3)
normalization phase.

\underline{Step 1:} Consensus Phase\\
In the first phase, the consensus belief $\bel^{\text{cons}}_{X}(x)$ is computed as a minimum common
belief per $x$.
The residual beliefs $\bel^{\text{res}A}_{X}(x)$ of each actor are
the differences between their belief and the consensus belief.
The total consensus $b_X^{\text{cons}}$ is the sum of all consensus
beliefs.

\begin{eqnarray}
  \bel^{\text{cons}}_{X}(x) &=& \min_{A\in\mathbb{A}} \bel^{A}_{X}(x) \\
  \bel^{\text{res}A}_{X}(x) &=& \bel^{A}_{X}(x) - \bel^{\text{cons}}_{X}(x)\text{, for each $A\in\mathbb{A}$} \nonumber \\
  b_X^{\text{cons}} &=& \sum_{x\in\mathcal{R}(\mathbb{X})} \bel^{\text{cons}}_{X}(x) \nonumber
\end{eqnarray}

\underline{Step 2:} Compromise Phase\\
The compromise belief $\bel^{\text{comp}}_X(x)$ in a frame $x$ is computed by adding four components: 1) the residue belief, weighted by all other actors' uncertainty, 2) the common belief in frame sets where the intersection is $x$, 3) the common belief in sets where the union is $x$, but whose intersection is non-empty, and 4) the common belief in the sets where the union is $x$ and the intersection is empty.
This can be formulated as follows:

\begin{multline}
    \bel^{\text{comp}}_X(x) = \sum\limits_{A\in\mathbb{A}} \bel^{\text{res}A}_X(x) \cdot \prod\limits_{A'\in\mathbb{A}, A' \neq A} u^{A'}_X  \\
     + \sum\limits_{\substack{y_1,\dots,y_{\lvert\mathbb{A}\rvert}\\\text{ s.t. }\cap_i y_i=x}} \prod\limits_{i=1}^{\lvert\mathbb{A}\rvert} \bel^{\text{res}A_i}_X(y_i) \cdot \base_X(y_i\mid y_j,\,j\neq i) \\
     + \sum\limits_{\substack{y_1,\dots,y_{\lvert\mathbb{A}\rvert}\\\text{ s.t. }\cup_i y_i=x \\\text{ and }\cap_i y_i\neq\emptyset}} \left(
        \left(1 - \prod\limits_{i=1}^{\lvert\mathbb{A}\rvert} \base_X(y_i\mid y_j,\,j\neq i)\right)
            \cdot \prod\limits_{i=1}^{\lvert\mathbb{A}\rvert} \bel^{\text{res}A_i}_X(y_i)
        \right)
    \\
     +\sum\limits_{\substack{y_1,\dots,y_{\lvert\mathbb{A}\rvert}\\\text{ s.t. }\cup_i y_i=x \\\text{ and }\cap_i y_i=\emptyset}} \prod\limits_{i=1}^{\lvert\mathbb{A}\rvert}\bel^{\text{res}A_i}_X(y_i) \text{ where } x\in\mathcal{P}(\mathbb{X})
\end{multline}

The second, third and forth terms of this equation taken together basically iterate over a tabulation of the different valid combinations for $y_i\in\mathcal{P}(\mathbb{X})$, and depending on the constraints defined under the sum, one of the three computations is made.
Since set operations are associative and commutative, the ordering of $\mathbb{A}$ implied by this equation is irrelevant, since all orderings are included in the tabulation.
This is analogous to the way this operator was initially defined by~J{\o}sang~\cite{Josang-SL-book}.

Additionally, in the second phase the preliminary uncertainty mass $u_X^{\text{pre}}$ is computed as the product of all uncertainties.
The intuition is that this represents the common uncertainty all actors agree on.
The value $b_X^{\text{comp}}$ is computed as sum of the compromise beliefs over $\mathcal{P}(\mathbb{X})$:

\begin{align}
  u_X^{\text{pre}} &= \prod_{A\in\mathbb{A}}u^{A}_X\\
  b_X^{\text{comp}} &= \sum_{x\in\mathcal{P}(\mathbb{X})} \bel^{\text{comp}}_X(x)
\end{align}

These values are used in the next phase.

\underline{Step 3:} Normalization Phase\\
Since mostly $b_X^{\text{cons}} + b_X^{\text{comp}} + u_X^{\text{pre}}
< 1$, a normalization factor $\eta$ of $b_X^{\text{comp}}$ has to be
applied.
We compute $\eta$ as follows.

\begin{equation}
  \eta = \frac{1 - b_X^{\text{cons}} - u_X^{\text{pre}}}{b_X^{\text{comp}}}
\end{equation}

The belief on the entire domain $\mathbb{X}$ is then added to the uncertainty.
The intuition behind this is that a belief in $\mathbb{X}$ is the belief that \emph{anything} will happen.
Referring back to the tournament example from Hankin, the belief that $\mathbb{X}$ will win is the belief that a winner exists, without distinguishing between the singleton values in this set, which is equivalent to non-informative belief, i.e., uncertainty.

\begin{equation}
  u_X^{\CC \mathbb{A}} = u_X^{\text{pre}} + \eta \bel_X^{\text{comp}}(\mathbb{X}).
\end{equation}

After this transfer, we set $\bel_X^{\text{comp}}(\mathbb{X}) = 0$.

Now, the belief can be combined over all $x\in\mathcal{R}(\mathbb{X})$
after applying the normalization factor.
\begin{equation}
    \bel^{\CC\mathbb{A}}_X(x) = \bel_X^{\text{cons}}(x) + \eta\bel^{\text{comp}}_X(x).
\end{equation}

\begin{definition}[CC-Fusion of multiple sources]
Given a finite set of actors $\mathbb{A}$ and their multinomial
opinions $\omega^{A}_{X}= (\bel^{A}_{X}, u^{A}_{X}, \base_{X})$
over $X$ for $A\in\mathbb{A}$.
We define the resulting CC-fused opinion as a result of the previous
three-step computation as
$\omega_X^{\CC\mathbb{A}} = (\bel_X^{\CC\mathbb{A}}, u_X^{\CC\mathbb{A}}, \base_X)$.
\end{definition}

As this fusion operation is defined for subjective logic only, we do not prove it's equivalence to other operations in other models.
However, we note that this operation always generates valid opinions, and that our definition reduces to the original definition for the case that $\lvert\mathbb{A}\rvert=2$.

\begin{remark}[Multi-source CCF always generates valid opinions]
  This holds if the output of multi-source CCF produces valid beliefs and uncertainties (in the range of $[0,1]$), and the additive property holds, i.e., $\sum_{x\in\mathcal{P}(\mathbb{X})}b_X(x) = 1$.
  The first property holds, since the definition consists only of additions and multiplications of non-negative numbers, except for the definition $\bel^{\text{res}A}_X(x)=b^A_X(x) - b^{\text{cons}}_X(x)$, which is at least zero, since the latter term is the minimum of all the beliefs in $x$.
  The second property holds due to the choice of $\eta$, and because $\bel_X^A(\emptyset)=0$ for well-defined belief functions:
\begin{eqnarray*}
    1&\stackrel{!}{=}& u^{\CC\mathbb{A}}_X + \sum\limits_{x\in\mathcal{R}(\mathbb{X})}\bel^{\CC\mathbb{A}}_X\\
  &=& u^{\text{pre}}_X + \eta\bel^{\text{comp}}_X(\mathbb{X}) + \sum\limits_{x\in\mathcal{R}(\mathbb{X})}\left(\bel^{\text{cons}}_X(x) + \eta\bel^{\text{comp}}_X(x)\right)\\
  &=& u^{\text{pre}}_X + b^{\text{cons}}_X + \eta b^{\text{comp}}_X\\
  &=& u^{\text{pre}}_X + b^{\text{cons}}_X + 1-b^{\text{cons}}_X - u^{\text{pre}}_X = 1\\
\end{eqnarray*}
\end{remark}


\begin{remark}[Multi-source CCF is a generalization of the CC fusion operator]

  With $\mathbb{A}=\{A,B\}$, the steps resolve to these statements are as follows:

\begin{eqnarray}
  \bel^{\text{cons}}_{X}(x) &=& \min (\bel^{A}_{X}(x), \bel^{B}_{X}(x))\\
  \bel^{\text{res}A}_{X}(x) &=& \bel^{A}_{X}(x) - \bel^{\text{cons}}_{X}(x)\\
  \bel^{\text{res}B}_{X}(x) &=& \bel^{B}_{X}(x) - \bel^{\text{cons}}_{X}(x)\\
  b_X^{\text{cons}} &=& \sum_{x\in\mathcal{R}(\mathbb{X})} \bel^{\text{cons}}_{X}(x)
\end{eqnarray}

\begin{multline}
  \bel^{\text{comp}}_X(x) = \bel^{\text{res}A}_X(x) u^B_X +\bel^{\text{res}B}_X(x) u^A_X \\
     + \sum\limits_{\substack{y_1,y_2\\\text{ s.t. } y_1 \cap y_2=x}} \bel^{\text{res}A}_X(y_1)  \base_X(y_1|y_2)  \base_X(y_2|y_1)\\
     + \sum\limits_{\substack{y_1,y_2\\\text{ s.t. } y_1 \cup y_2=x\\\text{ and } y_1 \cap y_2 \neq \emptyset}} (1-\base_X(y_1|y_2)\base_X(y_2|y_1)) \bel^{\text{res}A}_X(y_1)  \bel^{\text{res}B}_X(y_2)\\
     + \sum\limits_{\substack{y_1,y_2\\\text{ s.t. } y_1 \cup y_2=x\\\text{ and } y_1 \cap y_2 \neq \emptyset}} \bel^{\text{res}A}_X(y_1)\bel^{\text{res}B}_X(y_2)
\end{multline}

\begin{align}
  u_X^{\text{pre}} &= u^A_X u^B_X\\
  b_X^{\text{comp}} &= \sum_{x\in\mathcal{P}(\mathbb{X})} \bel^{\text{comp}}_X(x)
\end{align}

\begin{equation}
  \eta = \frac{1 - b_X^{\text{cons}} - u_X^{\text{pre}}}{b_X^{\text{comp}}}
\end{equation}

\begin{equation}
  u_X^{\CC \mathbb{A}} = u_X^{\text{pre}} + \eta \bel_X^{\text{comp}}(\mathbb{X})
\end{equation}

  and finally, setting $\bel^{\text{comp}}(\mathbb{X})=0$.

  Except for differences in notation, this is the same as the definitions provided in~\cite[Ch. 12]{Josang-SL-book}.
\end{remark}
%
%
%

\section{Example}
\label{sec:example}

\begin{table*}[ht!]
\centering
  \caption{Extending the examples from~\cite{Josang-Multi-Source}. The numbers presented here are rounded. The final row of the table presents the projection of each operation to the probabilistic space.}
\label{tab:example}
\begin{tabular}{@{}l|lll|llllll@{}}
    \toprule
    & \multicolumn{3}{c|}{Inputs} & \multicolumn{6}{c}{Fusion}\\
Parameters   & $A_1$ & $A_2$  & $A_3$ & aCBF       & \textbf{eCBF} & \textbf{BCF} & ABF    & \textbf{WBF} & \textbf{CCF} \\
\hline
$\bel(x)$       & 0.10  & 0.40   & 0.70  & 0.651      & 0.442         & 0.738        & 0.509  & 0.562        & 0.629        \\
$\bel(\bar{x})$ & 0.30  & 0.20   & 0.10  & 0.209      & 0             & 0.184        & 0.164  & 0.146        & 0.182        \\
$u$          & 0.60  & 0.40   & 0.20  & 0.140      & 0.558         & 0.078        & 0.327  & 0.292        & 0.189        \\
$\base(x)$       & 0.5   & 0.5    & 0.5   & 0.5        & 0.5           & 0.5          & 0.5    & 0.5          & 0.5          \\
    \midrule
$P(x)$       & 0.40  & 0.60   & 0.80  & 0.721      & 0.721         & 0.777        & 0.673  & 0.708        & 0.723        \\
    \bottomrule
\end{tabular}
\end{table*}

In this section, we extend the table from J{\o}sang~et~al.~\cite{Josang-Multi-Source} with results for the fusion operations we have defined.
This example considers a situation where three sources, $\mathbb{A}=\{A_1,A_2,A_3\}$, provide opinions over the same binary domain $\mathbb{X}=\{x,\bar{x}\}$.
These opinions can be merged using the operators we discussed in this paper: the numerical results are shown in Table~\ref{tab:example}.
An existing open-source implementation for binomial opinions was extended with these operations\footnote{\texttt{https://github.com/vs-uulm/subjective-logic-java}}.
We briefly discuss the functionality of each new operation (highlighted in bold in the table).

\begin{itemize}
  \item \textbf{epistemic Cumulative Belief Fusion} (eCBF) is the uncertainty-maximized version of aCBF.
    As previously mentioned, \emph{epistemic} here refers to the fact that the opinions involved are used for representing knowledge, rather than documenting specific observations.
    Epistemic opinions are always uncertainty maximized: they represent exactly the available information about $X$ beyond the base rate $\base$, and no more.
    This operation is useful for artificial reasoning about abstract events, which are not tied to a specific instance.
  \item \textbf{Belief Constraint Fusion} (BCF) is the generalization of Dempster's rule of combination that considers distinct base rates.
    The interpretation of BCF has been the subject of many works over the years~\cite{Zadeh1984,Josang-SL-book,Shafer2016}, most importantly noting that conflict is essentially discarded.
  \item \textbf{Weighted Belief Fusion} (WBF) is the evidence-weighted combination of belief from different sources.
    Similar to averaging belief fusion, it is useful where dependence between these sources is assumed.
    However, it introduces additional weighting, increasing the significance of sources that possess high certainty.
    This automatic weighting is useful for the combination of evidence from a large variety of uncertain sources which output high certainty in very specific scenarios.
    In such cases, ABF issues high uncertainty, while WBF outputs what is intuitively a consensus between experts: each expert individually chooses their confidence, and thus their weight in the decision.
  \item \textbf{Consensus \& Compromise Fusion} (CCF) is specifically designed to create vague belief from conflicting belief on singleton values.
    More precisely, generating a vague belief refers to the derivation of a belief of a composite value $x_1 \cap x_2$ if two opinions have high belief in $x_1$ and $x_2$ respectively.
    This is useful when different experts generate opinions identifying different options, such as when doctors with different expertise suggest potential causes in a diagnostic process.
    The fused opinion reflects the opinion of all experts, and illustrates the group as a whole is certain about a certain set of potential causes, without expressing relative likelihood.
\end{itemize}

Potential applications of these operations include their use for multi-source fusion between entities.
In our previous work, we have proposed the use of subjective logic for a misbehavior detection framework~\cite{Dietzel2014}, which we have since implemented.
For our analysis, multi-source fusion with subjective logic was required; in our future work, we intend to adopt the results presented here into this framework.

\balance

\section{Conclusion}
\label{sec:conclusion}

In this article, we have introduced multi-source fusion operations for several existing fusion operators: the belief constraint fusion, the weighted belief fusion, and the consensus \& compromise fusion.
Since the native operators' definitions are non-associative, they require a definition to enable fusion of information from multiple sources.
We also discuss the intuition behind the fusion process and the corresponding interpretation of the fusion result.
Finally, we have proven that WBF corresponds to the confidence-weighted averaging of evidence in the Dirichlet multinomial model.

In our future work, we aim to implement and apply these new multi-source fusion operations in our work, which is centered around attack detection through the fusion of information from different sources.
We currently foresee a combination of fusion operations with trust discounting and trust revision to update trust in various sources.
We intend to apply both WBF and CCF for this fusion process, but for different source types.

Finally, we would like to comment that both Dempster-Shafer Theory~\cite{Zadeh1984,Tchamova2012} and Subjective Logic~\cite{Dezert2014} have been the subject of controversy over the years.
The entirety of our work relies on the recently published book discussing subjective logic~\cite{Josang-SL-book}, which contains a clear map between subjective opinions and Dirichlet PDFs, resolving the main issue raised by Dezert~et~al.~\cite{Dezert2014} (namely that different mappings have been proposed, and it is unclear which is the valid one).
Nevertheless, the criticisms and alternative proposals made by these and other authors should be considered seriously.
To our best understanding, the concerns raised have been resolved sufficiently, which is also evidenced by the continued study and extension of subjective logic at major publication platforms in the field~\cite{Josang-Multi-Source}.
We argue that our proposed derivations of multi-source fusion operations is a useful contribution to the field.

\section*{Acknowledgment}

We would like to thank David K\"ohler, whose work with an implementation of subjective logic led to the discovery of the inconsistency discussed in \Cref{sec:CBF-ABF}, and Audun J{\o}sang for an extensive discussion that resulted in this paper.
This work was supported in part by the Baden-W\"{u}rttemberg Stiftung gGmbH Stuttgart as part of the project IKT-05 AutoDetect of its IT security research programme.

\bibliographystyle{IEEEtran}
\bibliography{IEEEabrv,references}

\begin{thebibliography}{10}
\providecommand{\url}[1]{#1}
\csname url@samestyle\endcsname
\providecommand{\newblock}{\relax}
\providecommand{\bibinfo}[2]{#2}
\providecommand{\BIBentrySTDinterwordspacing}{\spaceskip=0pt\relax}
\providecommand{\BIBentryALTinterwordstretchfactor}{4}
\providecommand{\BIBentryALTinterwordspacing}{\spaceskip=\fontdimen2\font plus
\BIBentryALTinterwordstretchfactor\fontdimen3\font minus
  \fontdimen4\font\relax}
\providecommand{\BIBforeignlanguage}[2]{{%
\expandafter\ifx\csname l@#1\endcsname\relax
\typeout{** WARNING: IEEEtran.bst: No hyphenation pattern has been}%
\typeout{** loaded for the language `#1'. Using the pattern for}%
\typeout{** the default language instead.}%
\else
\language=\csname l@#1\endcsname
\fi
#2}}
\providecommand{\BIBdecl}{\relax}
\BIBdecl

\bibitem{Shafer-DST}
G.~Shafer, \emph{A Mathematical Theory of Evidence}.\hskip 1em plus 0.5em minus
  0.4em\relax Princeton University Press, 1976.

\bibitem{Dempster-DST}
\BIBentryALTinterwordspacing
A.~P. Dempster, ``Upper and lower probabilities induced by a multivalued
  mapping,'' \emph{Ann. Math. Statist.}, vol.~38, no.~2, pp. 325--339, 04 1967.
  [Online]. Available: \url{https://doi.org/10.1214/aoms/1177698950}
\BIBentrySTDinterwordspacing

\bibitem{Josang-SL-book}
A.~J{\o}sang, \emph{Subjective Logic: A Formalism for Reasoning Under
  Uncertainty}, ser. Artificial Intelligence: Foundations, Theory and
  Algorithms.\hskip 1em plus 0.5em minus 0.4em\relax Springer International
  Publishing Switzerland, 2016.

\bibitem{Josang-Multi-Source}
A.~J{\o}sang, D.~Wang, and J.~Zhang, ``Multi-source fusion in subjective
  logic,'' in \emph{2017 20th International Conference on Information Fusion
  (Fusion)}.\hskip 1em plus 0.5em minus 0.4em\relax IEEE, July 2017, pp. 1--8.

\bibitem{Hankin2010}
R.~K.~S. Hankin, ``A generalization of the dirichlet distribution,''
  \emph{Journal of Statistical Software}, vol.~33, 2010.

\bibitem{Josang-belief-fusion}
A.~J{\o}sang, J.~Diaz, and M.~Rifqi, ``Cumulative and averaging fusion of
  beliefs,'' \emph{Information Fusion}, vol.~11, no.~2, pp. 192--200, Apr.
  2010.

\bibitem{Zadeh1984}
L.~A. Zadeh, ``Review of a mathematical theory of evidence,'' \emph{AI
  magazine}, vol.~5, no.~3, p.~81, 1984.

\bibitem{Shafer2016}
G.~Shafer, ``Dempster's rule of combination,'' \emph{International Journal of
  Approximate Reasoning}, vol.~79, pp. 26 -- 40, 2016, 40 years of Research on
  Dempster-Shafer theory.

\bibitem{Dietzel2014}
S.~Dietzel, R.~W. van~der Heijden, H.~Decke, and F.~Kargl, ``{A Flexible,
  Subjective Logic-based Framework for Misbehavior Detection in V2V
  Networks},'' in \emph{IEEE 15th International Symposium on a World of
  Wireless, Mobile and Multimedia Networks (WoWMoM)}, June 2014.

\bibitem{Tchamova2012}
A.~Tchamova and J.~Dezert, ``On the behavior of dempster's rule of combination
  and the foundations of dempster-shafer theory,'' in \emph{2012 6th IEEE
  International Conference Intelligent Systems}, Sept 2012, pp. 108--113.

\bibitem{Dezert2014}
J.~Dezert, A.~Tchamova, D.~Han, and J.-M. Tacnet, ``Can we trust subjective
  logic for information fusion?'' in \emph{17th International Conference on
  Information Fusion (FUSION)}.\hskip 1em plus 0.5em minus 0.4em\relax IEEE,
  2014, pp. 1--8.

\end{thebibliography}

\end{document}